\documentclass[12pt]{article}%
\usepackage{amsmath}
\usepackage{amsthm}
\usepackage{amsfonts}
\usepackage{amssymb}
\usepackage{graphicx}

\providecommand{\U}[1]{\protect \rule{.1in}{.1in}}
\providecommand{\U}[1]{\protect \rule{.1in}{.1in}}
\providecommand{\U}[1]{\protect \rule{.1in}{.1in}}
\def \Pr{\mathop{\rm Pr}}
\newtheorem{theorem}{Theorem}
\begin{document}

\title{Incorporating Expert Rules into Neural Networks in the Framework of Concept-Based Learning}
\author{Andrei V. Konstantinov and Lev V. Utkin\\
\small{Higher School of Artificial Intelligence Technologies}
\\
\small{Peter the Great St.Petersburg Polytechnic University}\\
\small{St.Petersburg, Russia}
\\
\small{e-mail: andrue.konst@gmail.com, lev.utkin@gmail.com}}
\date{}
\maketitle

\begin{abstract}
A problem of incorporating the expert rules into machine learning models for
extending the concept-based learning is formulated in the paper. It is
proposed how to combine logical rules and neural networks predicting the
concept probabilities. The first idea behind the combination is to form
constraints for a joint probability distribution over all combinations of
concept values to satisfy the expert rules. The second idea is to represent a
feasible set of probability distributions in the form of a convex polytope and
to use its vertices or faces. We provide several approaches for solving the
stated problem and for training neural networks which guarantee that the
output probabilities of concepts would not violate the expert rules. The
solution of the problem can be viewed as a way for combining the inductive and
deductive learning. Expert rules are used in a broader sense when any logical
function that connects concepts and class labels or just concepts with each
other can be regarded as a rule. This feature significantly expands the class
of the proposed results. Numerical examples illustrate the approaches. The
code of proposed algorithms is publicly available.


\end{abstract}

\section{Introduction}

Concept-based learning (CBL) is a well-established approach to express
predictions of a machine learning model in terms of high-level concepts
derived from raw features instead of in terms of the raw features themselves
\cite{lage2020learning}, i.e. unlike using a pixel-based level, CBL provides a
higher level of connection between the image and the decision using concepts.
The understanding of the decision becomes straightforward once the
interpretation of each concept is determined \cite{wang2023learning}.
High-level concepts can be interpreted as additional expert knowledge.
Therefore, CBL aims to integrate expert knowledge or human-like reasoning into
machine learning models. In the context of machine learning, incorporating
high-level concepts into the learning process may significantly improve the
efficiency and accuracy of models. Moreover, high-level concepts may improve
the explainability of the machine learning model outputs because they are
intuitive to users \cite{kim2018interpretability,yeh2020completeness}. On the
one hand, concepts can be viewed as high-level features of an object, for
example, a color of a bird in a picture. On the other hand, the same concepts
can be regarded as complex classification labels. The training of the concept
model often requires the concept annotations in the form of binary labels,
i.e. concept is \textquotedblleft present\textquotedblright \ or
\textquotedblleft not present\textquotedblright, for each defined concept and
image. However, concepts can be represented in other forms, for example, by
means of indices assigned to elements of a concept description set. It should
be noted that, indices can always be converted into binary concepts.

Many recent CBL approaches consider human-specified concepts as an
intermediate step to derive the final predictions. Concept Bottleneck Models
(CBMs) \cite{koh2020concept} and Concept Embedding Models (CEMs)
\cite{espinosa2022concept} are CBL models that implement these approaches.
According to the CBM approach, the CBL model provides the concept prediction
in the middle of the decision-making process, i.e. it explicitly predicts the
concept labels from images and then predicts the final label based on the
concept label predictions. The classifier deriving the final label has access
only to the concept representation, and the decision is strongly tied to the
concepts \cite{wang2023learning}. At that, the training procedure can be
implemented in an independent way when the concept labels are trained
independently on the final label. Another way for learning is to train in an
end-to-end manner.

In contrast to CBLs, we propose quite different models which can roughly be
called concept-based models because they use concepts for training like CBLs.
However, the main goal of the proposed model is to combine the inductive
learning tools (neural networks) with the knowledge-based expert rules of the
form \textquotedblleft IF ..., THEN ...\textquotedblright, which are elicited
from experts and constructed by means of concepts. For example, the rule from
the lung cancer diagnostics can look like \textquotedblleft IF
\textbf{Finding} is \emph{the mass}, \textbf{Contour} is \emph{spicules},
\textbf{Inclusion} is the \emph{air cavity}, THEN a \textbf{Diagnosis} is the
\emph{squamous cells carcinoma}\textquotedblright. Here concepts are shown in
Bold, the concept values are shown in Italic. Another illustrative example is
taken from the Bird identification dataset (CUB) \cite{Wah-etal-11}:
\textquotedblleft IF the \textbf{Head} is \emph{red}, the \textbf{Back color}
is \emph{black}, the \textbf{Breast color} is \emph{white}, the \textbf{Crown
color} is \emph{red}, the \textbf{Wing color} is \emph{white}, the
\textbf{Bill shape} may be \emph{dagger} OR \emph{all-purpose}, THEN the
\textbf{Bird} is a~\emph{red-headed woodpecker}\textquotedblright. In the
context of medicine, a doctor often diagnoses a disease based on certain rules
from a medical handbook. Using such rules is the basis of a doctor's work.
Similar examples of using expert rules can be found in various applied fields,
not just in medicine. Therefore, it is important to incorporate the expert
rules into machine learning models.

It is assumed that we have knowledge how final labels (consequents) of
instances depend on values of concepts (antecedents) from a set of
knowledge-based expert rules. Moreover, we have a partially labeled training
set consisting of images with some concept labels and with some targets which
will be called as final concepts. The question is how to construct and to
train a neural network which deals simultaneously with the concept-based
dataset and the knowledge-based rules to provide accurate predictions and to
explain the predictions by using concepts. In order to answer this question,
we propose two approaches to taking into account the expert rules.

First, we represent the knowledge-based expert rules in the form of logical
functions consisting of the disjunction and conjunction operations of
indicator functions corresponding to values of concepts. At that, the target
value is also represented as a concept. By having the logical functions, we
can write constraints for a joint probability distribution over all
combinations of concept values to satisfy the expert rules. This allows us to
construct and to train a neural network which guarantees that the output
probabilities of concepts would not violate the expert rules. We formulate the
corresponding feasible set of probability distributions in the form of a
convex polytope and analytically find its vertices. By means of the vertices,
a point inside the polytope can be constructed that determines marginal
probability distributions of concepts. Additionally, we can define the same
polytope in H-representation by setting its faces. It is useful because the
number of faces can be significantly smaller than the number of vertices.
These two ways to define the convex polytope form a base for developing four
approaches for constructing neural networks incorporating the expert rules.

An important peculiarity of the proposed models is that the expert rules
compensate the incomplete concept labeling of instances in datasets whereas
existing concept-based models may lead to overfitting when many images have
incomplete concept description. Moreover, the expert rules allow us to
compensate a partial availability of targets in the training set.

In sum, we try to incorporate the knowledge-based expert rules into a neural
network to improve predictions and their interpretation. The knowledge of
expert rules changes probabilities of concepts as well as predictions
corresponding to new instances which are classified and explained. The
proposed models can be viewed as a way for combining the inductive and
deductive learning.

It is also important to point out that the term \textquotedblleft expert
rules\textquotedblright \ is used in the proposed models not only to represent
the standard \textquotedblleft IF ...,THEN ...\textquotedblright \ rule, but in
a broader sense. Any logical function that connects concepts and class labels
or just concepts with each other can be regarded as a rule. This feature
significantly expands the class of the proposed results.

The code of proposed algorithms is available in https://github.com/andruekonst/ecbl.

\section{Related work}

\textbf{Concept-based learning models.} Many models taking into account
various aspects of CBL have been developed by following the works
\cite{kim2018interpretability,yeh2020completeness}. In particular, the concept
attribution approach to tabular learning by providing an idea on how to define
concepts over tabular data was proposed in \cite{pendyala2022concept}. An
algorithm for learning visual concepts directly from images, using
probabilistic predictions generated by visual classifiers as the input to a
Bayesian generalization model was proposed in \cite{jia2013visual}. A novel
concept-based explanation framework named Prototypical Concept-based
Explanation is proposed in \cite{dreyer2023understanding}. An idea of the
framework is that it provides differences and similarities to the expected
model behavior via prototypes which are representative predictions summarizing
the model behavior in condensed fashion. An analysis of correlations between
concepts and methods beyond the test accuracy for evaluating concept-based
models, with regard to whether a concept has truly been learned by the model
were presented in \cite{heidemann2023concept}.

Lage et al. \cite{lage2020learning} claim that many CBL models define concepts
which are not inherently interpretable. To overcome this limitation, the
authors proposed a CBL model where concepts are fully transparent, thus
enabling users to validate whether mental and machine models are aligned. An
important peculiarity of the model is that the corresponding learning process
incorporates user feedback directly when learning the concept definitions:
rather than labeling data, users mark whether specific feature dimensions are
relevant to a concept \cite{lage2020learning}. To relax an assumption that
humans are oracles who are always certain and correct in decision making
Collins et al. \cite{collins2023human} study how existing concept-based models
deal with uncertain interventions from humans. An attempt to suppress false
positive explanations by providing explanations based on statistically
significant concepts was carried out in \cite{xu2023statistically} where the
authors guarantee the reliability of the concept-based explanation by
controlling an introduced false discovery rate of the selected concepts.

Applications of the concept-based explanation in medicine can be found in
\cite{marcinkevivcs2024interpretable,Meldo-Utkin-etal-2020,patricio2023coherent,patricio2023towards,yan2023robust}%
. The use of CBL models for time-series data are presented in
\cite{obermair2023example,tang2020interpretable}. Taking into account the
anomaly detection problem, a framework for learning a set of concepts that
satisfy properties of the out-of-distribution detection and help to explain
the out-of-distribution predictions was presented in \cite{choi2023concept}.
In the same work, new metrics for assessing the effectiveness of a particular
set of concepts for explaining the out-of-distribution detection detectors
were introduced. The concept-based model for anomaly detection was also
considered in \cite{sevyeri2023transparent}.

Promises and pitfalls of black-box concept learning models
\cite{mahinpei2021promises}. A review of recent approaches for explaining
concepts in neural networks was provided in \cite{lee2023neural}.

\textbf{Concept bottleneck models. }Following \cite{koh2020concept}, many
extensions of the CBM model have been proposed. A part of the CBL models
belongs to post-hoc models. These models analyze the whole model only after it
has finished the training process. The post-hoc CBMs were introduced in
\cite{yuksekgonul2022post}. These models convert any pre-trained model into a
concept bottleneck model. This conversion can be done by using the concept
activation vectors (CAVs) \cite{kim2018interpretability} in a special way. A
Cooperative-CBM (coop-CBM) model was proposed in \cite{Sheth-Kahou-23}. The
model aims at addressing the performance gap between CBMs and standard
black-box models. It uses an auxiliary loss that facilitates the learning of a
rich and expressive concept representation. In order to take into account the
ambiguity in the concept predictions, a probabilistic CBM was proposed in
\cite{kim2023probabilistic}, which exploits probabilistic embeddings in the
concept embedding space and reflects uncertainty in the concept predictions.
The model maps an image to the concept embeddings with probabilistic
distributions which model concept uncertainties.

According to \cite{espinosa2022concept}, one of the drawbacks of many CBM
models is that they unable to find optimal compromises between high task
accuracy, robust concept-based explanations, and effective interventions on
concepts. In order to overcome this drawback, concept embedding models were
introduced in \cite{espinosa2022concept}. The models can be viewed as a family
of CBMs that represents each concept as a supervised vector, i.e. the models
learn two embeddings per concept, one for when it is active, and another when
it is inactive. Following the concept embedding model
\cite{espinosa2022concept}, the concept bottleneck generative models were
introduced in \cite{ismail2023concept}, where a concept bottleneck layer is
constrained to encode human-understandable features. Raman et al.
\cite{Raman-etal-24} studied whether CBMs correctly capture the degree of
conditional independence across concepts when the concepts are localized
spatially and semantically. Margeloiu et al. \cite{margeloiu2021concept}
demonstrated that concepts may not correspond to anything semantically
meaningful in input space. A simple procedure allowing to perform
concept-based instance-specific interventions on an already trained black-box
neural network is proposed in \cite{Marcinkevics-etal-24}.

A novel image representation learning method, called the Concept-based
Explainable Image Representation Learning and adept at harnessing human
concepts to bolster the semantic richness of image representations was
introduced in \cite{cui2023ceir}. A case of implicit knowledge corresponding
to the unsupervised (unlabeled) concepts was studied in
\cite{sawada2022concept} where the authors propose to adopt self-explaining
neural networks to obtain the unsupervised concepts. These networks are
composed of an encoder-decoder architecture and a parametrizer estimating
weights of each concept. Energy-based CBMs which use a set of neural networks
to define the joint energy of candidate (input, concept, class) tuples are
introduced in \cite{Xu-Qin-etal-24}.

CBMs were extended in \cite{chauhan2023interactive} to interactive prediction
settings such that the model can query a human collaborator for the label to
some concepts. In order to improve the final prediction, an interaction policy
was developed in \cite{chauhan2023interactive} that chooses which concepts
should be requested for labeling. An approach to modify CBMs for images
segmentation, objects fine classification and tracking was developed in
\cite{pittino2023hierarchical}. Two causes of performance disparity between
soft (inputs to a label predictor are the concept probabilities) and hard (the
label predictor only accepts binary concepts) CBMs were proposed in
\cite{havasi2022addressing}. They allow hard CBMs to match the predictive
performance of soft CBMs without compromising their resilience to leakage. A
similar task was solved in \cite{Sun-Yan-etal-24}. Marconato et al.
\cite{marconato2022glancenets} provided a definition of interpretability in
terms of alignment between the representation of a model and an underlying
data generation process, and introduced GlanceNets, a new CBM that exploits
techniques from disentangled representation learning and open-set recognition
to achieve alignment, thus improving the interpretability of the learned concepts.

Drawing inspiration from the CLIP model \cite{radford2021learning}, a
foundation model that establishes a shared embedding space for both text and
images, the CLIP-based CBMs are proposed in \cite{kazmierczak2023clip}.

\section{Background}

The concept-based classification is a task to construct a potentially
black-box classifier and to explain the constructed classifier's decision
process through human-interpretable concepts \cite{xu2023statistically}.

We are given a set of inputs $\mathbf{x}_{i}\in$ $\mathcal{X}\subset
\mathbb{R}^{d}$ and the corresponding targets $y_{i}\in \mathcal{Y}%
=\{1,2,...,K\}$. Suppose we are also given a set of $m$ pre-specified concepts
$\mathbf{c}_{i}=(c_{i}^{(1)},...,c_{i}^{(m)})\in \mathcal{C}$ such that the
training set comprises $(\mathbf{x}_{i},y_{i},\mathbf{c}_{i})$, $i=1,...,n$.
Typically, concepts can be represented as a binary $m$-length vector
$\mathbf{c}_{i}$ where its $j$-th element $c_{i}^{(j)}$denotes whether the
$j$-th concept is present or not in the input $\mathbf{x}_{i}$.

Generally, the CBL model aims to find how targets depend on concepts and
inputs, i.e., to find a function $h:(\mathcal{X},\mathcal{C})\rightarrow
\mathcal{Y}$. However, the CBL model can also be used to interpret how
predictions depend on concepts corresponding to inputs. In order to solve this
task, CBMs have been developed, which learn two mappings, one from the input
to the concepts $g:$ $\mathcal{X}\rightarrow \mathcal{C}$, and another from the
concepts to the outputs $f:\mathcal{C}\rightarrow \mathcal{Y}$. In this case,
the CBM prediction for a new input instance $\mathbf{x}$ is defined as
$y=f(g(\mathbf{x}))$.

There are different problem settings in the framework of CBL models
\cite{Xu-Qin-etal-24}:

\emph{Prediction}: Given the input $\mathbf{x}$, the goal is to predict the
class label $y$ and the associated concepts $\mathbf{c}$ to interpret the
predicted class label, that is to find the probability $\mathop{\rm Pr}
(\mathbf{c},y\mid \mathbf{x})$. Note that CBMs decompose $\mathop{\rm Pr}
(\mathbf{c},y\mid \mathbf{x})$ to predict $\mathop{\rm Pr} (\mathbf{c\mid x})$
and then $\mathop{\rm Pr} (y\mid \mathbf{x})$.

\emph{Concept Correction/Intervention}: Given the input $\mathbf{x}$ and a
corrected concept $c^{(k)}$, predict all the other concepts $c^{(i)}$,
$i=1,...,m,$ $i\neq k$.

\emph{Conditional Interpretations}: Given an image with class label $y$ and
concept $c^{(j)}$ , what is the probability that the model correctly predicts
concept $c^{(k)}$.

We propose a different problem setting which uses concepts to incorporate the
available expert rules of the form: \textquotedblleft IF concepts have certain
values, THEN the target is equal to a certain class\textquotedblright, into
neural networks.

\section{Expert rules and concepts}

\subsection{Problem statement}

In contrast to the above definition of concepts as binary variables, which is
conventional in many models, we assume that each concept $c^{(i)}$ can take
one of $n_{i}$ values denoted as $\mathcal{C}^{(i)}=\{1,...,n_{i}\}$,
$i\in \{0,\dots,m\}$. We call $\mathcal{C}^{(i)}$ the $i$-th concept outcome
set. The concept vector is $\mathbf{c}=(c^{(0)},c^{(1)},...,c^{(m)}%
)\in \mathcal{C}^{\times}$, where $\mathcal{C}^{\times}$ is the concept domain
produced by the Cartesian product $\mathcal{C}^{\times}=\mathcal{C}%
^{(0)}\times \dots \times \mathcal{C}^{(m)}$. We consider the concept $c^{(0)}$
as a special concept corresponding to the target variable $y$.

Let us introduce the logical literal $h_{i}^{(j)}(\mathbf{c})=\mathbb{I}$
$[c^{(j)}=i]$ which takes the value $1$, if the concept $c^{(j)}$ has the
value $i$. A set of expert rules is formulated as a logical expression
$g(\mathbf{c})$ over literals $h_{i}^{(j)}(\mathbf{c})$. Formally a set of
expert rules may be represented as a mapping $g:\mathcal{C}^{\times}%
\mapsto \{0,1\}$, where $0$ means FALSE, and $1$ means TRUE.

For example, the rule \textquotedblleft IF $c^{(1)}=3$ THEN $c^{(2)}%
=1$\textquotedblright \ is equivalent to the function:
\begin{equation}
g=(c^{(1)}=3)\rightarrow(c^{(2)}=1)=h_{3}^{(1)}\rightarrow h_{1}^{(2)}=(\lnot
h_{3}^{(1)})\vee h_{1}^{(2)},
\end{equation}
where $\rightarrow$ is implication and $\lnot$ is negation and the argument
$\mathbf{c}$ is omitted for short. A literal negation is equivalent to
disjunction of the rest outcomes of the same concept. For example, $\lnot
h_{3}^{(1)}\equiv h_{1}^{(1)}\vee h_{2}^{(1)}$, when $\mathcal{C}%
^{(1)}=\{1,2,3\}$.

Let $X$ be a random vector taking values from $\mathcal{X} \subset
\mathbb{R}^{a}$. We introduce $C^{(0)}, \dots, C^{(m)}$ as discrete random
variables for the concepts, taking values from $\mathcal{C}^{(0)}, \dots,
\mathcal{C}^{(m)} $, respectively. The concept random vector is $C =
(C^{(0)},C^{(1)},...,C^{(m)})$.

The task is to estimate marginal concept probabilities
$\mathop{\rm Pr}(C^{(i)}=j\mid X=\mathbf{x})$ conditioned on an input object
$\mathbf{x}$ (for example, image) for the $i$-th concept and outcome $j$ under
condition of the expert rules.

For brevity, we denote the marginal concept probabilities as vectors:
\begin{equation}
p^{(i)}=\left(  \mathop{\rm Pr}(C^{(i)}=1\mid \mathbf{x}),\dots
,\mathop{\rm Pr}(C^{(i)}=n_{i}\mid \mathbf{x})\right)  ,
\end{equation}
which is standard for classification. Since the marginal probabilities are not
independent because of the expert rules, we cannot estimate $p^{(i)}$
separately. Instead, the goal is to find the concatenated vector of marginal
concept probabilities:
\begin{equation}
\overline{p}=[(p^{(0)})^{T};\dots,(p^{(m)})^{T}]^{T}.
\end{equation}

%The next question is how to incorporate the rules into a neural network
%trained on the dataset $\mathcal{D}$ consisting of triplets $(\mathbf{x}%
%_{i},y_{i},\mathbf{c}_{i})$, $i=1,...,N$, or pairs $(\mathbf{x}_{i}%
%,\mathbf{c}_{i})$, $i=1,...,N$, to predict the probability vectors
%$\langle \Pr(C^{(i)} = j \mid \mathbf{x}) \rangle_{j=1}^{n_i}$
%for a new instance $\mathbf{x}$ and each concept $i$.

\subsection{Joint distribution}

First, we consider the conditional joint probability distribution over
concepts: $\mathop{\rm Pr}(C=\mathbf{c}|X=\mathbf{x})$. Since all concept
random variables are discrete with finite outcome sets, all possible vectors
$\mathbf{c}$ can be enumerated. The total number of distinct concept vectors
is $t=\prod_{i=0}^{m}n_{i}$. Let us introduce a function $\mathcal{M}%
:\mathcal{C}^{\times}\mapsto \{1,\dots,t\}$, that maps the concept vector to
its number, and its inverse function $\mathcal{M}^{-1}$ that maps the number
to the concept vector. We define the joint probability distribution vector
$\pi=(\pi_{1},\dots,\pi_{t})$ as follows:%
\begin{equation}
\forall \mathbf{c}\in \mathcal{C}^{\times},\  \  \pi_{\mathcal{M}(\mathbf{c}%
)}=\mathop{\rm Pr}(C=\mathbf{c}|X=\mathbf{x}).
\end{equation}

\subsection{Reduction to a linear constraint}

The joint probability distribution is constrained to satisfy the expert rules
formulated as $g$, therefore:
\begin{equation}
\mathop{\rm Pr}(g(C)=1)=1.\label{eq:rule_proba_constraint}%
\end{equation}

Let us consider a binary mask of admissible states, a vector $u=(u_{1}%
,\dots,u_{t})\in \{0,1\}^{t}$, whose components are equal to $1$ if and only if
the rules are satisfied for the corresponding concept vectors:%
\begin{equation}
u_{k}=g(\mathcal{M}^{-1}(k)),\  \ k\in \{1,...,t\}.
\end{equation}

The constraint on the joint probability distribution
\eqref{eq:rule_proba_constraint} can be reformulated as a set of the equality
constraints on components of $\pi$, corresponding to invalid states (that
violate the rules):
\begin{equation}
\pi_{k}=0,\  \ k\in \{i\in \{1,...t\} \mid g(\mathcal{M}^{-1}%
(k))=0\}.\label{eq:eq_to_zero_system}%
\end{equation}

The vector $\pi$ also obeys the following probability distribution
constraints:
\begin{equation}
\pi \in \Delta_{t},
\end{equation}
so the system \eqref{eq:eq_to_zero_system} can be rewritten in more compact
form as one linear equality constraint:
\begin{equation}
u^{T}\pi=1.\label{eq:u_pi_one}%
\end{equation}

Here $\Delta_{t}$ is the unit simplex of dimension $t$.

To construct a feasible solution, a neural network should be able to generate
probability distributions matching equality constraint.

For illustrative purposes, we consider a toy example with two classes of
\textbf{birds}: a \emph{red-headed woodpecker}\ ($c^{(0)}=1$) and an
\emph{European green woodpecker}\ ($c^{(0)}=2$). The corresponding concepts
describing the birds are \textbf{head}\ ($c^{(1)}$), \textbf{bill
shape}\ ($c^{(2)}$), and \textbf{wing color}\ ($c^{(3)}$). Concept
\textbf{head}\ can take values: \emph{red}\ ($c^{(1)}=1$), \emph{green
}($c^{(1)}=2$), concept \textbf{bill shape }can take values: \emph{chisel
}($c^{(2)}=1$), \emph{dagger }($c^{(2)}=2$), \emph{all-purpose} \ ($c^{(2)}%
=3$). Here $m=2$, $n_{0}=2$, $n_{1}=2$, $n_{2}=3$. Suppose there is available
the following expert rule:
\begin{equation}%
\begin{array}
[c]{llll}%
\text{IF} & \text{ \textbf{head}} & \text{is} & \text{\emph{red}}\\
\text{AND} & \text{ \textbf{bill shape}} & \text{is} & \text{\emph{dagger} OR
\emph{all-purpose},}\\
\text{THEN} & \text{ \textbf{bird}} & \text{is} & \text{\emph{red-headed
woodpecker},}%
\end{array}
\end{equation}
or
\begin{equation}%
\begin{array}
[c]{ll}%
\text{IF} & c^{(1)}=1\text{ AND }c^{(2)}\in \{2,3\},\\
\text{THEN} & c^{(0)}=1\text{.}%
\end{array}
\end{equation}

This rule can be represented as follows:
\begin{align}
g(\mathbf{c})  &  =\left(  h_{1}^{(1)}\wedge \left(  h_{2}^{(2)}\vee
h_{3}^{(2)}\right)  \right)  \rightarrow h_{1}^{(0)}\nonumber \\
&  =h_{1}^{(0)}\vee \lnot\left(  h_{1}^{(1)}\wedge \left(  h_{2}%
^{(2)}\vee h_{3}^{(2)}\right)  \right) \nonumber \\
&  =h_{1}^{(0)}\vee \lnot h_{1}^{(1)}\vee \left(  \lnot h_{2}%
^{(2)}\wedge \lnot h_{3}^{(2)}\right) \nonumber \\
&  =h_{1}^{(0)}\vee h_{2}^{(1)}\vee h_{1}^{(2)}.\label{woodpecker_3}%
\end{align}

Table \ref{t:probab_comb} shows in bold all possible combinations of the
concept values satisfying the above expert rule. It can be seen from Table
\ref{t:probab_comb} that there holds:
\begin{equation}
\pi_{1}+\pi_{2}+\pi_{3}+\pi_{4}+\pi_{5}+\pi_{6}+\pi_{7}++\pi_{10}+\pi_{11}%
+\pi_{12}=1.
\end{equation}
%

%TCIMACRO{\TeXButton{B}{\begin{table}[tbp] \centering}}%
%BeginExpansion
\begin{table}[tbp] \centering
%EndExpansion
\caption{An example of combinations of the concept values and the
corresponding probabilities}%
\begin{tabular}
[c]{ccccccccccccc}\hline
$c^{(0)}$ & $\mathbf{1}$ & $\mathbf{1}$ & $\mathbf{1}$ & $\mathbf{1}$ &
$\mathbf{1}$ & $\mathbf{1}$ & $\mathbf{2}$ & $2$ & $2$ & $\mathbf{2}$ &
$\mathbf{2}$ & $\mathbf{2}$\\ \hline
$c^{(1)}$ & $\mathbf{1}$ & $\mathbf{1}$ & $\mathbf{1}$ & $\mathbf{2}$ &
$\mathbf{2}$ & $\mathbf{2}$ & $\mathbf{1}$ & $1$ & $1$ & $\mathbf{2}$ &
$\mathbf{2}$ & $\mathbf{2}$\\ \hline
$c^{(2)}$ & $\mathbf{1}$ & $\mathbf{2}$ & $\mathbf{3}$ & $\mathbf{1}$ &
$\mathbf{2}$ & $\mathbf{3}$ & $\mathbf{1}$ & $2$ & $3$ & $\mathbf{1}$ &
$\mathbf{2}$ & $\mathbf{3}$\\ \hline
$\mathbf{\pi}$ & $\pi_{1}$ & $\pi_{2}$ & $\pi_{3}$ & $\pi_{4}$ & $\pi_{5}$ &
$\pi_{6}$ & $\pi_{7}$ & $\pi_{8}$ & $\pi_{9}$ & $\pi_{10}$ & $\pi_{11}$ &
$\pi_{12}$\\ \hline
\end{tabular}
\label{t:probab_comb}%
%TCIMACRO{\TeXButton{E}{\end{table}}}%
%BeginExpansion
\end{table}%
%EndExpansion

\subsection{Probabilistic approach}

An alternative approach is to consider the joint probability under condition
that expert rules are satisfied:
\begin{equation}
\mathop{\rm Pr} \left(  C=\mathbf{c}\mid g(C)=1\right)  =\frac{\mathop{\rm Pr}
\left(  C=\mathbf{c},~g(C)=1\right)  }{\mathop{\rm Pr} \left(  g(C)\right)  }.
\end{equation}

The probability of conjunction $\mathop{\rm Pr} \left(  C=\mathbf{c}%
,~g(C)=1\right)  $ can be expanded as
\begin{equation}
\mathop{\rm Pr}(C=\mathbf{c},g(C)=1)=\mathop{\rm Pr}(C=\mathbf{c}%
)\cdot \mathop{\rm Pr}(g(C)=1|C=\mathbf{c}),
\end{equation}
where $\mathop{\rm Pr}(C=\mathbf{c})$ is a predicted probability distribution
that may not satisfy the expert rules.

The posterior probability depends on the deterministic function $g$, thus
\begin{equation}
\mathop{\rm Pr}(g(C)=1\mid C=\mathbf{c})=%
\begin{cases}
1, & \text{if}~g(\mathbf{c})=1\\
0, & \text{else}.
\end{cases}
\end{equation}

Let us find the probability $\mathop{\rm Pr} \{g(C)\}$ as follows:%
\begin{equation}
\mathop{\rm Pr}(g(C))=\sum_{\mathbf{k}\in \mathcal{C}^{\times}}\mathbb{I}%
[g(\mathbf{k})]\cdot \mathop{\rm Pr}(C=\mathbf{k}).
\end{equation}

The prior joint concept probability can be calculated as an output of a neural
network. Let us denote the output $\widehat{\pi}$:
\begin{equation}
\mathop{\rm Pr}(C=\mathbf{c})=\widehat{\pi}_{\mathcal{M}(\mathbf{c})}.
\end{equation}

Hence, there holds%
\begin{equation}
\pi_{j}=\frac{\widehat{\pi}_{j}\cdot \mathbb{I}[g(\mathcal{M}^{-1}(j))]}%
{\sum_{\mathbf{k}\in \mathcal{C}^{\times}}\widehat{\pi}_{k}\cdot \mathbb{I}%
[g(\mathcal{M}^{-1}(k))]}.
\end{equation}

In sum, this approach produces a mask for admissible probabilities $\pi
_{1},...,\pi_{N}$. The mask is the same as $u$ in \eqref{eq:u_pi_one}. This
approach requires to predict all components of the joint probability
distribution by applying a neural network, while only admissible states will
be used. So, it is quite redundant, and this motivates us not to use the
approach in practice. However, it is flexible and can be useful, for example,
in case when multiple conflicting expert rules are applied to different parts
of one dataset, or when the choice of expert rules depends on input.

\subsection{Solution set as a polytope}

Instead of modelling the whole joint distribution, one can estimate
probabilities only of the states that lead to satisfying the expert rules. We
call this reduced vector as the \textquotedblleft admissible probability
vector\textquotedblright \ and denote as $\tilde{\pi}=(\tilde{\pi}_{1}%
,\dots,\tilde{\pi}_{d})$, where $d$ is the number of admissible states. There
is no additional constraints on $\tilde{\pi}$, therefore we can say that it
belongs to the unit simplex of dimension $d$.

The joint probability vector can be found as
\begin{equation}
\pi=W\tilde{\pi},\label{eq:pi_W_tilde_pi}%
\end{equation}
where $W\in \{0,1\}^{t\times d}$ is a placement matrix which contains strictly
one non-zero element in each column and one or zero non-zero elements in each
row. It can be interpreted as arrangement of $\tilde{\pi}$ entries to the
admissible components of $\pi$.

The desired marginal concept probabilities can be calculated as a summation
over relevant entries of $\pi$:
\begin{equation}
\mathop{\rm Pr}(C^{(i)}=j\mid \mathbf{x})=\sum_{c\in \mathcal{C}^{\times}}%
\pi_{\mathcal{M}(c)}\cdot \mathbb{I}[c^{(i)}=j].
\end{equation}

So each vector $p^{(i)}$ can be represented as
\begin{equation}
p^{(i)}=B^{(i)}~\pi=B^{(i)}W\tilde{\pi},\label{eq:p_B_pi}%
\end{equation}
where $B_{jk}^{(i)}=\mathbb{I}\left[  \left(  \mathcal{M}^{-1}(k)\right)
^{(i)}=j\right]  $.

Then every solution satisfying rules can be expressed as
\begin{equation}
\begin{gathered} \overline{p} = V \tilde \pi, \\ V = [(B^{(0)} W)^T; \dots; (B^{(m)} W)^T]^T, \\ \tilde \pi \succcurlyeq 0, ~ \mathbf{1}^T \tilde \pi = 1. \end{gathered}\label{eq:V_def}%
\end{equation}
Therefore the solution set is by definition a polytope whose vertices are
columns of the matrix $V$.

In practice, this approach can be used as follows. First, all possible concept
vectors are enumerated and passed through the expert rules, represented as $g$
to obtain the mapping $W$ to the admissible states. Then $\tilde{\pi}$ is
obtained as an output of a neural network after applying the \emph{softmax}
operation, that is $\tilde{\pi}$ is formally a discrete probability
distribution. Then the solution is a point inside the polytope, which is
calculated by weighing columns of the vertex matrix $V$ with elements of
$\tilde{\pi}$.

The main disadvantage of this approach is that we have to pre-calculate and
store all vertices, while their amount can be enormous for complex logical
expressions on multiple concepts.

\subsection{Linear inequality system}

The solution set is a polytope with a possibly high number of vertices. We
discover an alternative definition of this set as an intersection of
half-spaces, determined by hyperplanes, the so-called H-representation.
Instead of converting from V- to H-representation after calculating vertices,
we construct a linear inequality system from scratch based only on $g$.

The algorithm for constructing the linear inequality system consists of three steps:

\begin{enumerate}
\item Convert $g$ into conjunctive normal form (CNF).

\item Map each clause to exactly one linear inequality.

\item Unify (intersect) clause's linear inequalities into one system, along
with the probability distribution constraints on marginals.
\end{enumerate}

The first step is NP-complete in a general case, but can be solved in
reasonable time in many practical applications. We stress here that the
algorithm is appropriate when the expert rules can be converted to a compact
CNF. Let the rule set be represented in CNF as a conjunction of clauses:
\begin{equation}
\mathcal{R}=\mathcal{K}_{1}\wedge \dots \wedge \mathcal{K}_{b}.
\end{equation}

Each clause $\mathcal{K}_{l}$ is a disjunction of literals:
\begin{equation}
\mathcal{K}_{l}=\bigvee_{q\in K_{l}}l_{q},
\end{equation}
where $q=(i,j)$, $l_{q}=h_{j}^{(i)}$ for some set of literals $K_{l}$ of the
clause. Note, that if the clause contains some negated literals like $\lnot
h_{j}^{(i)}$, they are replaced with the disjunction of the rest outcome
literals for the $i$-th concept:
\begin{equation}
\lnot h_{j}^{(i)}\equiv \bigvee_{k\in \mathcal{C}^{(i)}\setminus \{j\}}%
h_{k}^{(i)}.
\end{equation}

Further, we assume that such the transformation was applied to all clauses,
and then each clause does not contain negations.

The next steps can be applied to any boolean expressions in CNF. Let us
describe them in detail. First, consider a clause $\mathcal{K}$, whose
literals are like $h_{j}^{(i)}$. The goal is to find appropriate marginals
$\overline{p}$, for which some corresponding joint probability distribution
satisfying the clause exists. Formally, given a clause of the form:
\begin{equation}
\mathcal{K}=\bigvee_{q\in K}l_{q},
\end{equation}
the sum of probabilities of dependent literals has a lower bound:
\begin{equation}
\mathop{\rm Pr}(\mathcal{K})=1\implies \sum_{q}\mathop{\rm Pr}(l_{q}%
)\geq \mathop{\rm Pr}(\mathcal{K})=1.
\end{equation}

This property can be used to formulate a constraint of marginal probabilities.
For the clause $\mathcal{K}$, the constraint is
\begin{equation}
\sum_{(i,j)\in K}p_{j}^{(i)}\geq1.
\end{equation}

The lower bound is tight in a sense that there exist feasible marginals that
sum up to one. Moreover, no other constraints (except the probability
distribution constraints on marginals) restrict the feasible set. It means
that, for any solution $\overline{p}$ satisfying the constraint (the sum of
different marginals probabilities is not less than $1$), there exists at least
one joint probability distribution matching rules that have the same marginal
probability distributions.

We apply this transformation to each clause to obtain one linear inequality
constraint per clause. The last step is to correctly merge the constraints of
the clauses into one system. Hopefully, it can be achieved by intersecting the
obtained inequalities.

\begin{theorem}
Given a rule $\mathcal{R}$ in CNF, consisting of $b$ clauses:
\begin{equation}
\mathcal{R}=\mathcal{K}_{1}\wedge \dots \wedge \mathcal{K}_{b},
\end{equation}
the constraint on the expert rule probability is equivalent to intersection of
clause constraints:
\begin{equation}
\mathop{\rm Pr}(\mathcal{R})=1\iff%
\begin{cases}
\mathop{\rm Pr}(\mathcal{K}_{1})=1,\\
\dots \\
\mathop{\rm Pr}(\mathcal{K}_{b})=1.
\end{cases}
\end{equation}

\end{theorem}

\begin{proof}
Necessity.
% $(\Rightarrow)$\\
\begin{equation*}
\Pr(\mathcal{R}) = 1 \implies \forall j\in \overline{1,b} ~~ \Pr(\mathcal{K}_j) \cdot
\Pr(\bigwedge_{i\neq j} \mathcal{K}_i | \mathcal{K}_j) = 1
\implies
\Pr(\mathcal{K}_j) = 1.
\end{equation*}
Sufficiency.
% $(\Leftarrow)$\\
	\begin{align}
	\forall j\in \overline{1, b}  ~~  \Pr(\mathcal{K}_j) = 1 \iff \Pr(\overline{\mathcal{K}_j}) = 0
	\implies \\
	0 = \sum_{i=1}^{b} \Pr(\overline{\mathcal{K}_i}) \ge
\Pr(\overline{\mathcal{K}_1} \lor \dots \lor \overline{\mathcal{K}_b}) =
	 \Pr(\overline{\mathcal{K}_1 \land \dots \land \mathcal{K}_b}) \ge 0
	 \implies \\
	 0 = \Pr(\overline{\mathcal{K}_1 \land \dots \land \mathcal{K}_b}) =
	 1 - \Pr(\mathcal{K}_1\land \cdots \land \mathcal{K}_b)
	 \implies \\
	 \Pr(\bigwedge_{i=1}^{b} \mathcal{K}_i) = 1.
	\end{align}
\end{proof}

Finally, according to the proposition and the theorem, when only marginal
distributions are of interest, any set of expert rules can be equivalently
transformed to a linear inequalities system of the form:
\begin{equation}%
\begin{cases}
\sum_{(i,j)\in K_{1}}p_{j}^{(i)}\geq1\\
\dots \\
\sum_{(i,j)\in K_{b}}p_{j}^{(i)}\geq1,
\end{cases}
\end{equation}
where $K_{r}$ is a set of concept-values pairs for a clause $r$ of the set of
rules in CNF. Lets call the matrix of the system as $\hat{A}$, thus the system
is:
\begin{equation}
\hat{A}~\overline{p}\geq \mathbf{1}.
\end{equation}

The entire system of constraints for marginal distributions includes also the
probability distribution constraints:
\begin{equation}
p^{(i)}\in \Delta_{n_{i}},
\end{equation}
and becomes:
\begin{align}
A~\overline{p}  &  \geq b,\label{28}\\
Q~\overline{p}  &  =\mathbf{1},
\end{align}
where
\begin{equation}
A=%
\begin{bmatrix}
\hat{A}\\
E
\end{bmatrix}
,~~b=%
\begin{bmatrix}
\mathbf{1}\\
\mathbf{0}%
\end{bmatrix}
,
\end{equation}%
\begin{equation}
Q=%
\begin{bmatrix}
1\dots1 & 0\dots0 & 0\dots0\\
0\dots0 & 1\dots1~ & 0\dots0\\
& \ddots & \\
0\dots0 & 0\dots0~ & 1\dots1\\
&  &
\end{bmatrix}
\in \mathbb{R}^{m+1\times s}.
\end{equation}

Note that dimensionality of the linear inequality system may be reduced from
$s$ to $(s-m)$ along with elimination of equality constraints because, for
each concept, strictly one entry can be removed by using the condition
$p_{1}^{(i)}=1-\sum_{j\neq1}p_{j}^{(i)}$.

\section{Neural network and expert rules}

Consider a partially-labeled multi-label multi-class classification problem.
The training dataset $\mathcal{D}$ consists of $N$ tuples $(\mathbf{x}%
_{j},\zeta_{j}^{(0)},\dots,\zeta_{j}^{(m)})$, where $\zeta_{j}^{(i)}%
\in \mathcal{C}^{(i)}\cup \{-1\}$ is a label of the $i$-th concept of the $j$-th
training observation. The label $\zeta_{j}^{(i)}$ is assigned to $-1$ if it is
unknown for the $j$-th observation. The main target $y_{j}$ is denoted as the
$0$-th concept $\zeta_{j}^{(0)}$ and can also be partially labeled, that is
$\zeta_{j}^{(0)}$ can be equal to $-1$.

We consider neural networks that simultaneously predict the marginal
distribution for each concept. For the $i$-th concept, the prediction mapping
is denoted as
\begin{equation}
f^{(i)}:\mathcal{X}\mapsto \Delta_{n_{i}},
\end{equation}
however, a neural network $f_{\theta}$ with parameters $\theta$ computes
$f_{\theta}^{(0)},\dots,f_{\theta}^{(m)}$ simultaneously.

The training loss function is a weighted sum of the masked cross entropy
losses over each concept
\begin{align}
\mathcal{L}  &  =\sum_{i=0}^{m}\omega^{(i)}\cdot \mathcal{L}^{(i)},\\
\mathcal{L}^{(i)}  &  =-\sum_{j=1}^{N}\mathbb{I}\left[  \zeta_{j}^{(i)}%
\neq-1\right]  \cdot \left(  \sum_{k=1}^{n_{i}}\mathbb{I}[\zeta_{j}%
^{(i)}=k]\cdot \log{(f^{(i)}(\mathbf{x}))_{k}}\right)  ,
\end{align}
where $\omega^{(i)}$ is a weight of the $i$-th concept loss, which is an
inverse number of labeled concept samples by default. The summation in
brackets is the log-likelihood for the $i$-th concept.

The neural network consists of classical layers (fully-connected or
convolutional layers, depending on a solved problem) with one special layer at
the end of the neural network, which we call as a \textbf{concept head}. This
layer maps an embedding produced by the preceding layers to the marginal class
probabilities and guarantees that they will satisfy expert rules for any
input. The approaches described above can be used to construct different
concept heads. Let us consider them in detail.

\subsection{Base head}

The most simple concept head implementation is to calculate prior joint
probability distribution vector $\widehat{\pi}$ using \emph{softmax} applied
to a linear layer that maps embedding to $t$ logits. Then, to satisfy the
expert rules, the posterior joint probability conditioned on the rules is
calculated by multiplying of admissible states by the mask $u$, and
renormalized. This approach does extra computation when calculating
probabilities of invalid states.
%but possibly can be useful
%when the expert rules change depending on input.

The Base Head approach is schematically shown in the top picture in Fig.
\ref{f:concept_2}.

\subsection{Admissible state head}

The idea of the Admissible State head (\textbf{AS-head}) is to estimate the
probability distribution on \textbf{only admissible} outcomes $\tilde{\pi}$
instead of the whole joint probability vector $\pi$. The full joint
probability vector is constructed using \eqref{eq:pi_W_tilde_pi}, and the
marginals are calculated by using the matrix multiplication \eqref{eq:p_B_pi}.
In a software implementation, we obtain marginal probabilities $p^{(i)}$ by
reshaping the flat vector $\pi$ into a multi-dimensional array with dimensions
$(n_{0},\dots,n_{m})$, and then summing up over all dimensions except $i$.
Formally, it is equivalent to \eqref{eq:p_B_pi}.

To obtain the placement matrix $W$ and the dimension $d$ of $\tilde{\pi}$,
before constructing such the layer, one needs to enumerate all possible joint
outcomes and evaluate the expert rules on them. It is not a problem for a
small number of concepts, when enumeration can be carried out in a reasonable
time. But it can be a problem when the number of concepts and their outcomes
is large enough. Therefore, additional optimizations are required in this case.%

%TCIMACRO{\FRAME{ftbpFU}{4.0165in}{4.7183in}{0pt}{\Qcb{The Base Head (the top
%picture) and the AS-Head (the bottom picture) approaches}}{\Qlb{f:concept_2}%
%}{concept_2.png}{\special{ language "Scientific Word";  type "GRAPHIC";
%width 4.0165in;  height 4.7183in;  depth 0pt;  original-width 8.9354in;
%original-height 10.5087in;  cropleft "0";  croptop "1";  cropright "1";
%cropbottom "0";  filename 'Concept_2.png';file-properties "XNPEU";}}}%
%BeginExpansion
\begin{figure}
[ptb]
\begin{center}
\includegraphics[
%%=10.508700in,
%%=8.935400in,
height=4.7183in,
width=4.0165in
]%
{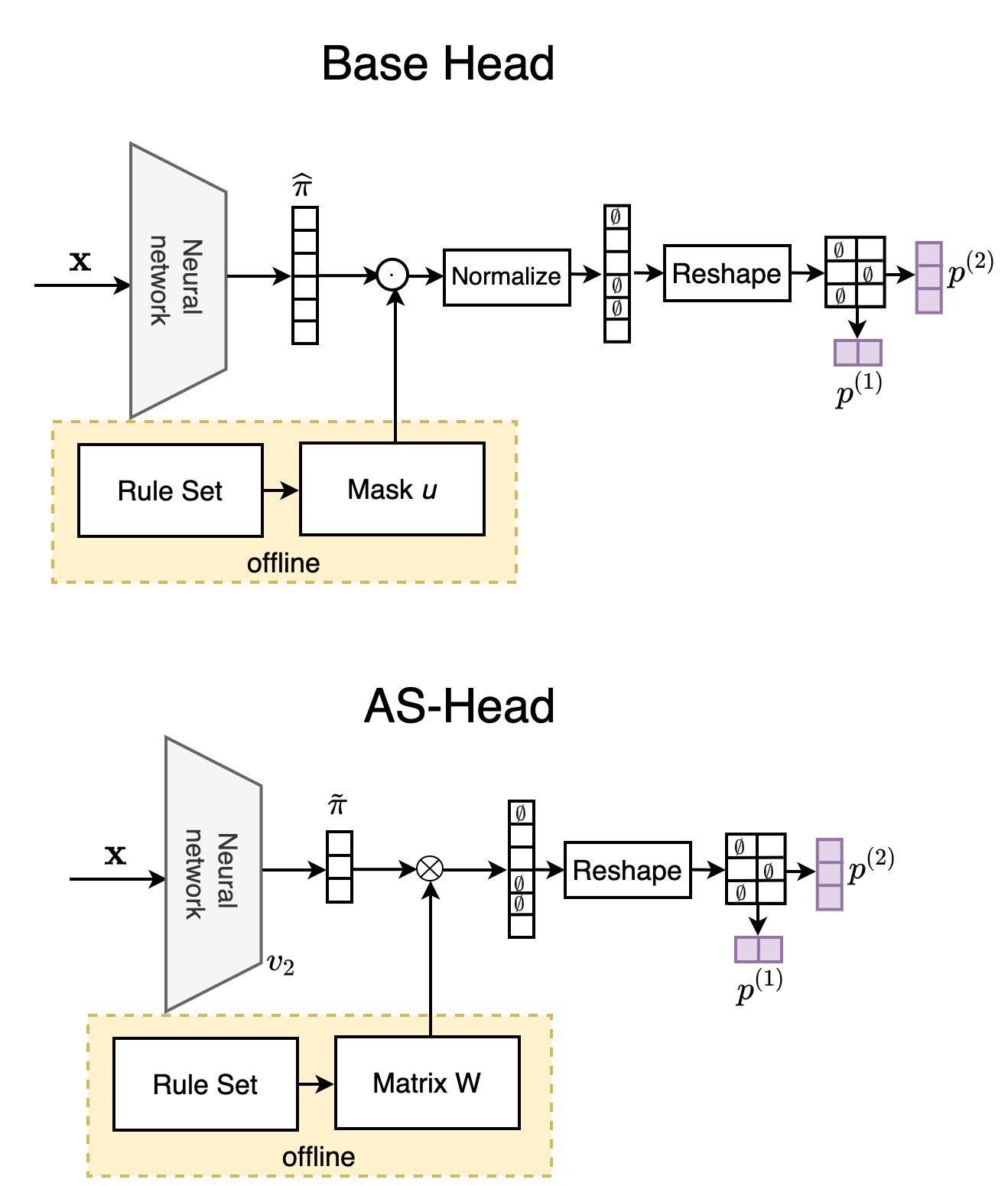}%
\caption{The Base Head (the top picture) and the AS-Head (the bottom picture)
approaches}%
\label{f:concept_2}%
\end{center}
\end{figure}
%EndExpansion

The bottom picture in Fig. \ref{f:concept_2} illustrates the Admissible State
Head approach.

\subsection{Vertex-based head}

The placement matrix $W$ is of dimension $t\times d$, where $t$ is the total
number of joint outcomes and $d$ is the number of valid states. Even if $d$
remains small, $t$ grows exponentially with the number of concepts $m$. Since
the main goal is to compute only marginal distributions using
\eqref{eq:V_def}, one can precompute the polytope vertex matrix $V$ of
dimension $s\times d$, where $s=\sum_{i=0}^{m}n_{i}$ is the dimension of the
vector $\overline{p}$ of concatenated marginal probabilities.

The computation of $V$ is carried out offline, before training a neural
network. Total number of operations at training or inference is reduced in
this case when dense or sparse matrices are used for storing $W,V$. If dense
matrices are used, then the number of operations is reduced exponentially
w.r.t. $m$ compared to the approaches described above.

The left picture in Fig. \ref{f:concept_1} illustrates a scheme of the Vertex
Head approach. The neural network generates the vector $\widetilde{\pi}$ in
the unit simplex, which is multiplied by the simplex vertices.

Another way for implementing the vertex-based approach is to first construct
the linear inequality system which defines the polytope of feasible solutions
$\overline{p}$. Then the vertices $V$ can be found via H- to V-representation
conversion \cite{fukuda2008exact}.

\subsection{Constraints head}

Alternatively, the solution $\overline{p}$ can be generated inside the
polytope defined by the linear inequality constraints, in H-representation
without estimation of vertices. The vertex-based head is an approach to
generate a point in a polytope by multiplying its vertices by weights from
\emph{softmax}. However there are other methods for such a problem, considered
in \cite{Konstantinov-Utkin-23f}. These methods require one feasible point as
an input that is strictly inside the polytope and can map an input embedding
into a polytope point. The methods have the computational complexity
$O(\nu \cdot \mu)$, where $\nu$ is the number of inequality constraints, $\mu$
is the output dimension.

The right picture in Fig. \ref{f:concept_1} illustrates a scheme of the
Constraints Head approach.

The main advantage of this approach is that it can be applied even when the
heads described above fail: the number of admissible states is enormous, the
matrix $V$ cannot be computed in a reasonable time, or uses too much memory,
in a couple with the matrix constructed from weights of the preceding linear
layer. It is because the number of inequality constraints $b$ (the number of
clauses in CNF) can be relatively small comparing to the number of admissible
joint outcomes $d$. The computational complexity of the point construction
inside the polytope, defined by the inequality constraints (\ref{28}), at
inference is $O\left(  (b+s)\cdot s\right)  $ under condition that at least
one point inside the polytope is found in advance. Therefore, it is reasonable
to apply this layer only when $d\gg b+s$.%

%TCIMACRO{\FRAME{ftbpFU}{6.1651in}{2.7035in}{0pt}{\Qcb{The Vetex Head (the left
%picture) and Constraints Head (the right picture) approaches}}%
%{\Qlb{f:concept_1}}{concept_1.png}{\special{ language "Scientific Word";
%type "GRAPHIC";  maintain-aspect-ratio TRUE;  display "USEDEF";
%valid_file "F";  width 6.1651in;  height 2.7035in;  depth 0pt;
%original-width 13.5596in;  original-height 5.925in;  cropleft "0";
%filename 'Concept_1.png';file-properties "XNPEU";}}}%
%BeginExpansion
\begin{figure}
[ptb]
\begin{center}
\includegraphics[
%%=5.925000in,
%%=13.559600in,
height=2.7035in,
width=6.1651in
]%
{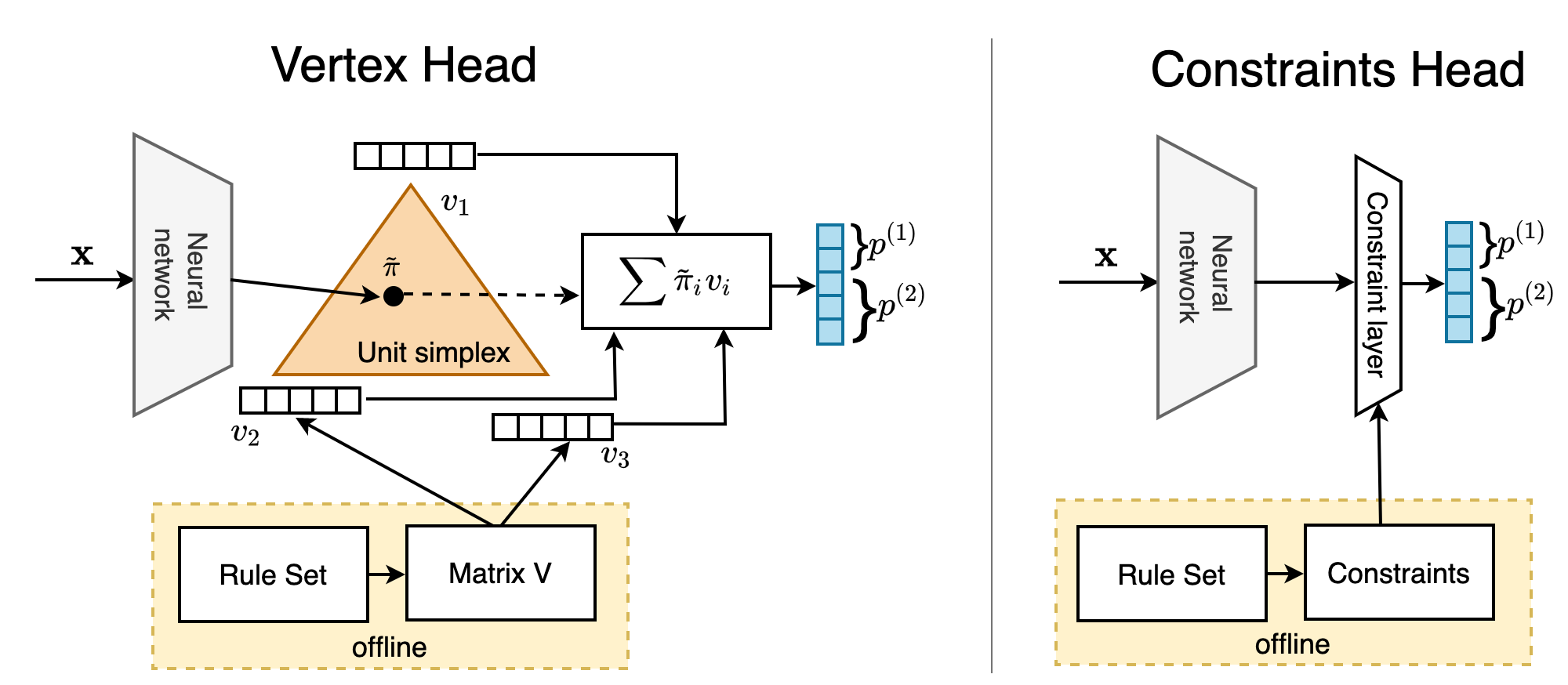}%
\caption{The Vetex Head (the left picture) and Constraints Head (the right
picture) approaches}%
\label{f:concept_1}%
\end{center}
\end{figure}
%EndExpansion

\subsection{State space reduction}

The number of enumeration steps, vertices or half-planes, depending on the chosen approach from the described above, can be strongly reduced by considering only concept values that are mentioned in expert rules. For this case, we construct
first separate concepts, that were not mentioned in the rules at all, and use
separate classification heads for them.

Other concepts, which are partially mentioned in the expert rules, can be
compressed. For this, we consider all values that were not mentioned in the
rules and incorporate them into a special $0$ outcome. For example, if rules
are based only on literals $h_{1}^{(1)},h_{2}^{(3)},h_{4}^{(3)}$, and
$\mathcal{C}^{(1)}=\{1,2\},\mathcal{C}^{(3)}=\{1,2,3,4\}$, then the compressed
outcome sets will be $\tilde{\mathcal{C}}^{(1)}=\{0,1\},\tilde{\mathcal{C}%
}^{(3)}=\{0,2,4\}$. Here for the first concept $c^{(1)}$, the outcome $2$ is
not used in the rules, therefore, we replace it with the artificial $0$
outcome. For the third concept $c^{(3)}$, outcomes $\{1,3\}$ are not used and
replaced with the outcome $0$.

After such the transformation, the total number of the joint distribution
outcomes is much less than the initial one. To infer probabilities of
compressed outcomes, we construct the additional classification heads that
estimate the probability distribution over outcomes which were replaced with
$0$, for each concept, that was partially mentioned in rules and has at least
two values for replacement. The final probability of compressed outcomes is
calculated as a multiplication of $0$ outcome probability by estimated
probabilities of replaced outcomes. In the above example, we have:
\begin{align}
&  \mathop{\rm Pr}(C^{(3)}=1)=\mathop{\rm Pr}_{\text{comp}}(C^{(3)}%
=0)\cdot \mathop{\rm Pr}_{\text{repl}}(C^{(3)}=1),\\
&  \mathop{\rm Pr}(C^{(3)}=2)=\mathop{\rm Pr}_{\text{comp}}(C^{(3)}=2),\\
&  \mathop{\rm Pr}(C^{(3)}=3)=\mathop{\rm Pr}_{\text{comp}}(C^{(3)}%
=0)\cdot \mathop{\rm Pr}_{\text{repl}}(C^{(3)}=3),\\
&  \mathop{\rm Pr}(C^{(3)}=4)=\mathop{\rm Pr}_{\text{comp}}(C^{(3)}=4),
\end{align}
where $\mathop{\rm Pr}_{\text{comp}}$ are the compressed probabilities,
$\mathop{\rm Pr}_{\text{repl}}$ are probabilities of replaced values estimated
by the separate classification heads.

\section{Numerical experiments}

\subsection{A toy example}

The first example is entirely synthetic. Two-dimensional input vectors are
randomly generated in the square $[0,1]^{2}$. The concepts are: $y=c^{(0)}%
\in \mathcal{C}^{(0)}=\{1,2\}$, $\mathcal{C}^{(1)}=\{1,2\}$, $\mathcal{C}%
^{(2)}=\{1,2,3\}$, $\mathcal{C}^{(3)}=\{1,2,3\}$.

Concepts used in the example are illustrated in Fig. \ref{f:toy_data}. In
particular, the first concept $c^{(1)}$ is equal to $2$ at the right from the
middle, and to $1$ at the left. The second concept $c^{(2)}$ is equal to $1$
at the bottom horizontal stripe of height $0.25$, to $2$ at the middle
horizontal stripe of height $0.5$ and to $3$ at the top stripe. The third
concept is like the second, but in the \textquotedblleft L\textquotedblright%
-shape, that is it depends on both features $x^{(1)}$ and $x^{(2)}$. The main
target $y$ is equal to $2$ if and only if $c^{(1)}=c^{(2)}=2$.%

%TCIMACRO{\FRAME{ftbpFU}{2.7786in}{2.7786in}{0pt}{\Qcb{Concepts used in the toy
%synthetic example}}{\Qlb{f:toy_data}}{toy_data.png}%
%width 2.7786in;  height 2.7786in;  depth 0pt;  original-width 6.0001in;
%original-height 6.0001in;  cropleft "0";  croptop "1";  cropright "1";
%cropbottom "0";  filename 'toy_data.png';file-properties "XNPEU";}}}%
%BeginExpansion
\begin{figure}
[ptb]
\begin{center}
\includegraphics[
%%=6.000100in,
%%=6.000100in,
height=3in,
]%
{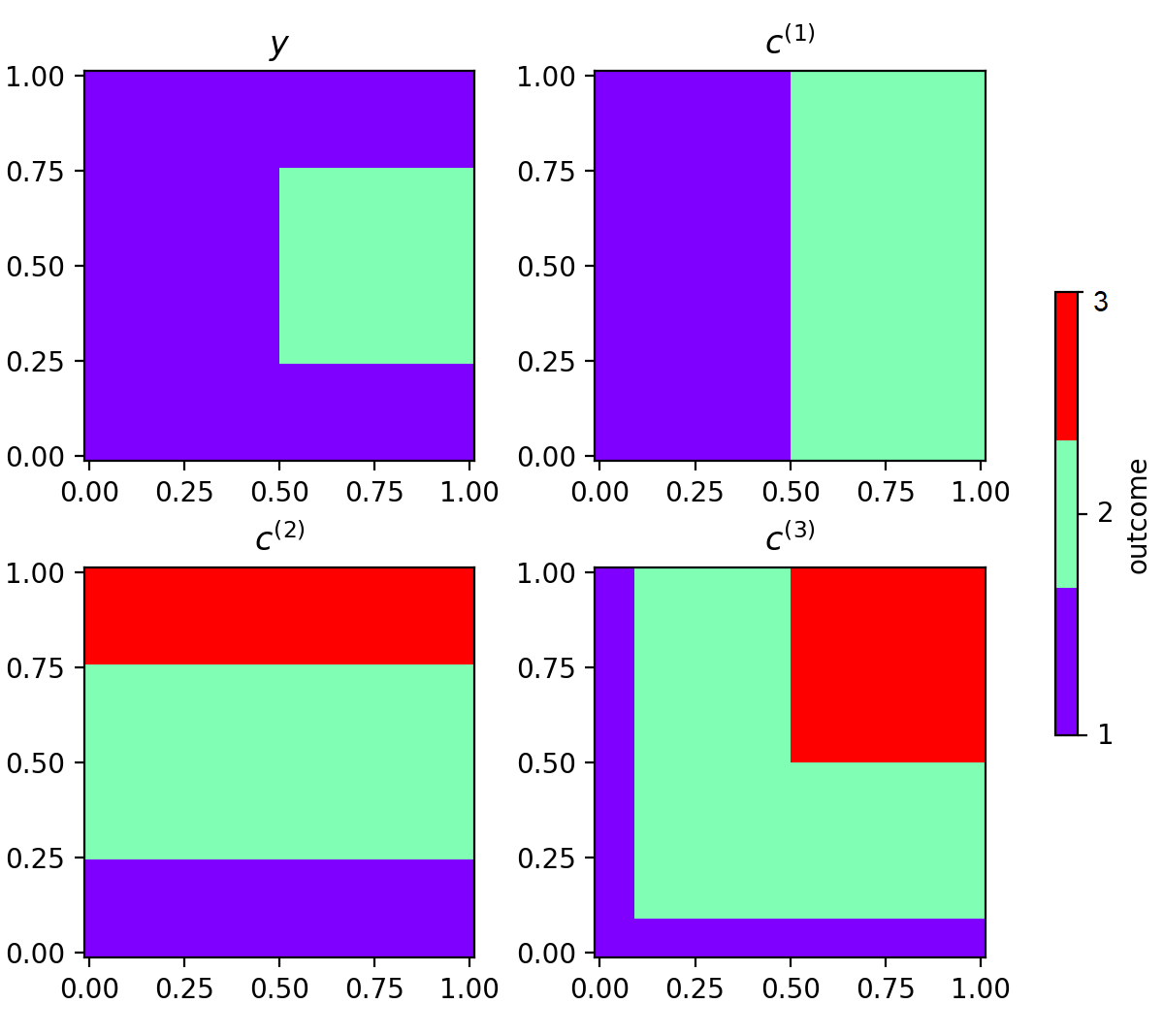}%
\caption{Concepts used in the toy synthetic example}%
\label{f:toy_data}%
\end{center}
\end{figure}
%EndExpansion

For example, let us consider the rule $g(\mathbf{c})=((c^{(1)}=2)\wedge
(c^{(2)}=2))\rightarrow(y=2)$ which is correct for the dataset. It can be
expressed with literals as $g(\mathbf{c})=(h_{2}^{(1)}\wedge h_{2}%
^{(2)})\rightarrow h_{2}^{(0)}$.
%The rule in CNF is
%$g(c) = (h^{(0)}_1 \lor h^{(1)}_2) \land (h^{(0)}_1 \lor h^{(2)}_2)$.
%The linear inequalities system is:
%\begin{equation*}
%\begin{cases}
%p^{(0)}_1 + p^{(1)}_2 \ge 1 \\
%p^{(0)}_1 + p^{(2)}_2 \ge 1 \\
%p^{(i)}_j \ge 0, & i \in \overline{0, m}, ~~ j \in \overline{1, n_i} \\
%\sum_{j=1}^{n_i} p^{(i)} = 1, & i \in \overline{0, m}.
%\end{cases}
%\end{equation*}

First, we train the model on the dataset completely without $y$ labels, i.e.
$\zeta_{k}^{(0)}=-1$ for each sample $k$. The model is able to reconstruct $y
$ by using the rules. The predicted probabilities are shown in Fig.
\ref{f:toy_probas_implication}. Second, we train the model for the same
dataset but with \textquotedblleft if and only if\textquotedblright \ rule
$g(\mathbf{c})=h_{2}^{(0)}\leftrightarrow(h_{2}^{(1)}\wedge h_{2}^{(2)})$. The
predicted probabilities are shown in Fig. \ref{f:toy_probas}. Note, that the
shape of $\mathop{\rm Pr}(y=2)$ is fully determined by the predicted
$\mathop{\rm Pr}(c^{(1)}=2)$ and $\mathop{\rm Pr}(c^{(2)}=2)$, while
$\mathop{\rm Pr}(c^{(3)}=2)$ does not affect $y$.%

%TCIMACRO{\FRAME{ftbpFU}{2.8651in}{2.8651in}{0pt}{\Qcb{Predicted probabilities
%with the implication rule}}{\Qlb{f:toy_probas_implication}}%
%{toy_probas_implication.png}{\special{ language "Scientific Word";
%type "GRAPHIC";  maintain-aspect-ratio TRUE;  display "USEDEF";
%valid_file "F";  width 2.8651in;  height 2.8651in;  depth 0pt;
%original-width 6.0001in;  original-height 6.0001in;  cropleft "0";
%filename 'images/toy_probas_implication.png';file-properties "XNPEU";}}}%
%BeginExpansion
\begin{figure}
[ptb]
\begin{center}
\includegraphics[
%%=6.000100in,
%%=6.000100in,
height=3in,
]%
{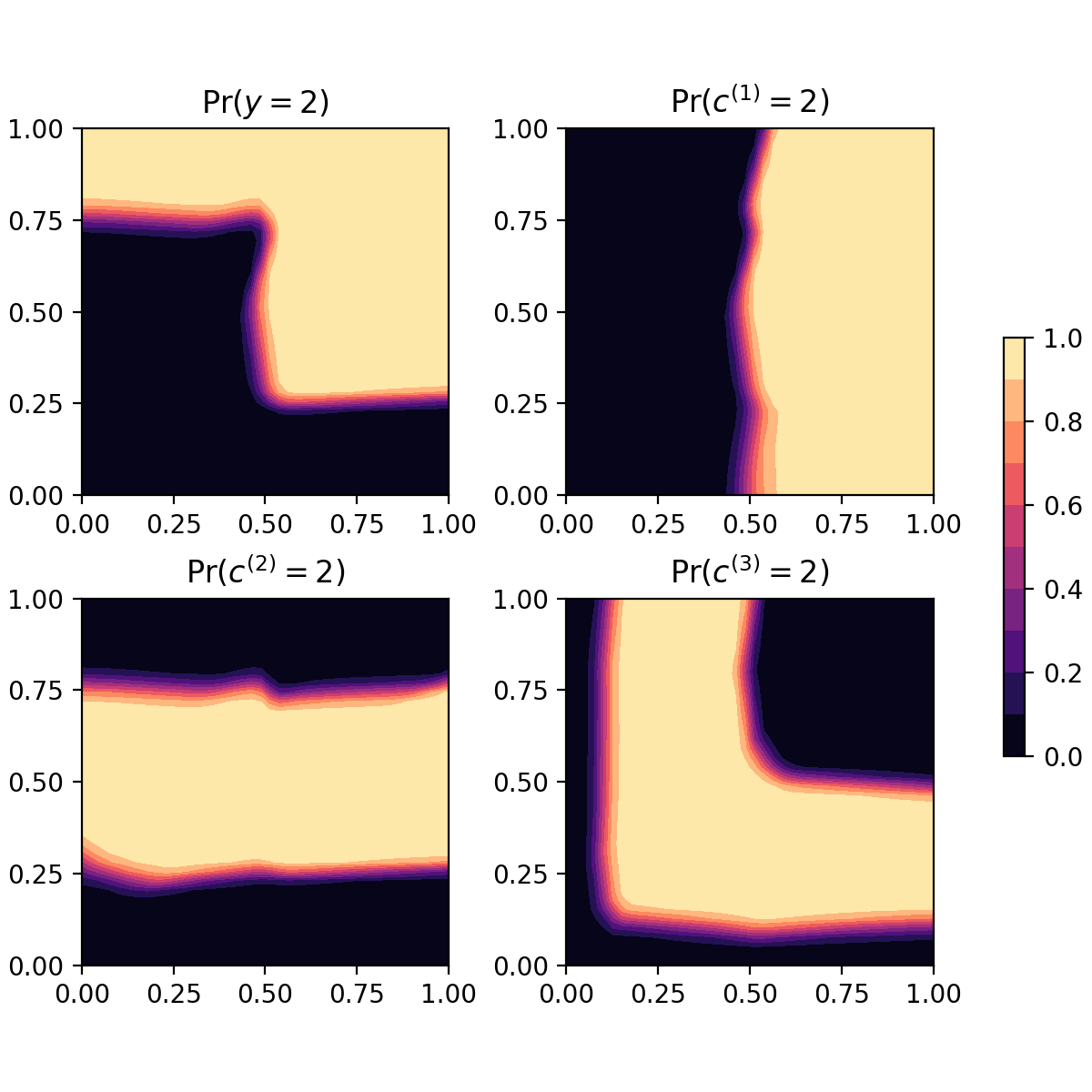}%
\caption{Predicted probabilities with the implication rule}%
\label{f:toy_probas_implication}%
\end{center}
\end{figure}
%EndExpansion
%

%TCIMACRO{\FRAME{ftbpFU}{2.8911in}{2.8911in}{0pt}{\Qcb{Predicted probabilities
%with the \textquotedblleft if and only if\textquotedblright\ rule}%
%}{\Qlb{f:toy_probas}}{toy_probas.png}{\special{ language "Scientific Word";
%type "GRAPHIC";  maintain-aspect-ratio TRUE;  display "USEDEF";
%valid_file "F";  width 2.8911in;  height 2.8911in;  depth 0pt;
%original-width 6.0001in;  original-height 6.0001in;  cropleft "0";
%filename 'images/toy_probas.png';file-properties "XNPEU";}}}%
%BeginExpansion
\begin{figure}
[ptb]
\begin{center}
\includegraphics[
%%=6.000100in,
%%=6.000100in,
height=3in,
]%
{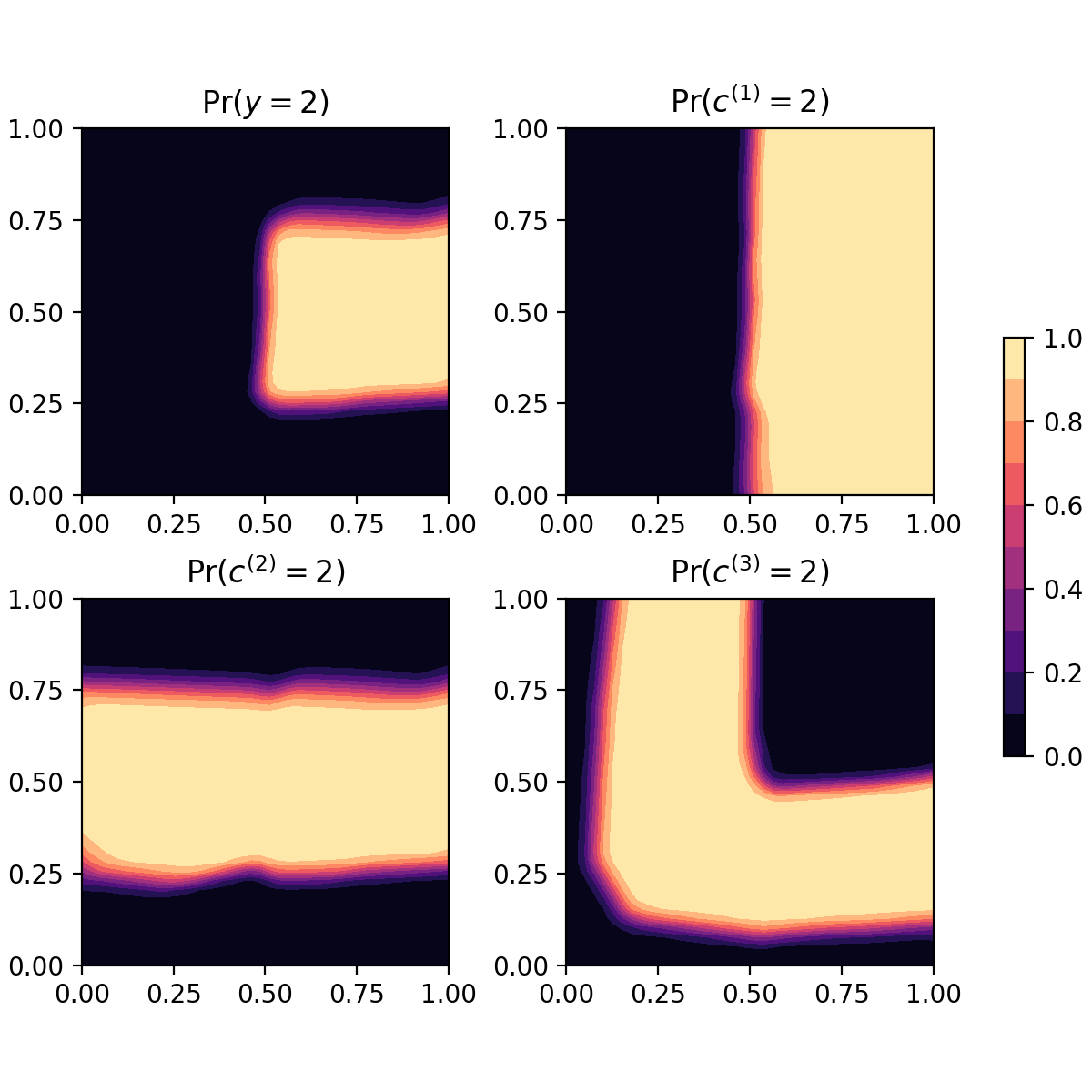}%
\caption{Predicted probabilities with the \textquotedblleft if and only
if\textquotedblright \ rule}%
\label{f:toy_probas}%
\end{center}
\end{figure}
%EndExpansion

\subsection{Multi-label MNIST}

The second example is with an artificial dataset constructed on a part of the
real labeled handwritten digits image dataset MNIST consisting of 5000
randomly selected images. We consider digit labels as the concept $c^{(1)}$
(not the main target). An additional synthetic feature is the digit color.
Each digit is randomly colored in white or blue corresponding to $c^{(2)}=1$
or $c^{(2)}=2$, respectively. The main label $y=c^{(0)}$ is defined as the
following: odd blue digits or even white digits are assigned to $y=1$, the
rest are assigned to $y=2$.

We compare three different types of heads: AS-Head, Joint Distribution Head
(AS-Head without expert rules) and Independent Classification Heads without
rules. The Independent Classification Heads model differs from the first two
models. It considers concepts as independent targets (classes) and predicts
probabilities for the targets separately. The first head is given the same
rule that was used for constructing labels for $y$. The second head predicts a
joint distribution and then calculates marginals from it. The third head, a
baseline, is a plain multi-label multi-class classification head, where
probabilities for each head are predicted independently. Results are shown in
Fig. \ref{f:mnist_5000_samples}, where the $F_{1}$ measure as a function of
the labeled data ratio is provided. The higher the labeled data fraction, the
easier is the task, therefore all three curves almost coincide at the fraction
of $0.5$. However when only a small amount of labeled data is available, the
proposed AS-Head, that takes expert rules into account, performs
systematically better than the baseline.%

%TCIMACRO{\FRAME{ftbpFU}{5.1924in}{1.7396in}{0pt}{\Qcb{Test performance
%($F_{1}$) depending on labeled data ratio}}{\Qlb{f:mnist_5000_samples}%
%}{mnist_5000_samples.png}{\special{ language "Scientific Word";
%type "GRAPHIC";  display "USEDEF";  valid_file "F";  width 5.1924in;
%height 1.7396in;  depth 0pt;  original-width 15.9999in;
%original-height 3.9998in;  cropleft "0";  croptop "1";  cropright "1";
%filename 'images/mnist_5000_samples.png';file-properties "XNPEU";}}}%
%BeginExpansion
\begin{figure}
[ptb]
\begin{center}
\includegraphics[
%%=3.999800in,
%%=15.999900in,
height=1.7396in,
]%
{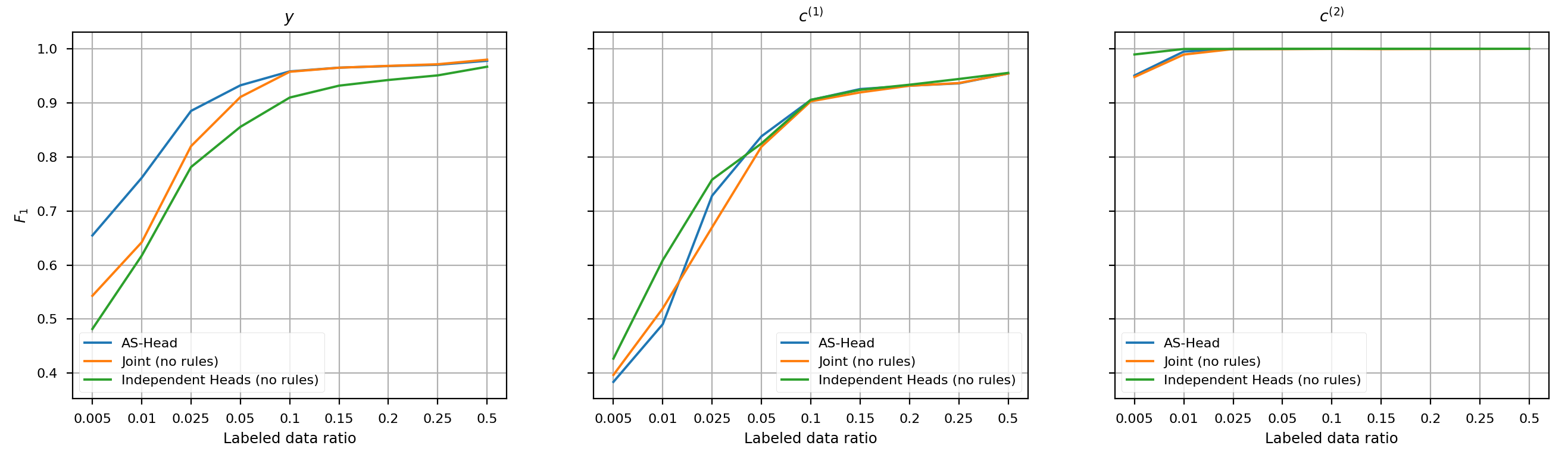}%
\caption{Test performance ($F_{1}$) depending on the labeled data ratio}%
\label{f:mnist_5000_samples}%
\end{center}
\end{figure}
%EndExpansion

\section{Conclusion}

We formulated the problem of incorporating the expert rules into machine
learning models for extending the concept-based learning for the first time.
We have shown how to combine logical rules and neural networks predicting the
concept probabilities. Several approaches have been proposed in order to solve
the stated problem and to implement the idea behind the use of expert rules in
machine learning. The proposed approaches have provided ways of constructing
and training a neural network which guarantees that the output probabilities
of concepts satisfy the expert rules. These ways are based on representing
sets of possible probability distributions of concepts by means of a convex
polytope such that the use of its vertices or its faces allows the neural
network to generate a probability distribution of concepts satisfying the
expert rules.

It has been illustrated by the numerical examples that the proposed models
compensate the incomplete concept labeling of instances in datasets. Moreover,
the expert rules allow us to compensate a partial availability of targets in
the training set.

The proposed approaches have different computational complexity depending on
the number of concepts, the number of concept values, and the number of
training examples.

The general problem of incorporating the expert rules into neural networks has
been solved. However, there are problems where an application of the proposed
models can significantly improve the accuracy and interpretability of models.
In particular, it is interesting to adapt the proposed models to CBMs which
also deal with concepts and can be combined with the expert rules. This is an
important direction for further research.

\bibliographystyle{unsrt}
\bibliography{Classif_bib,Concept,MYBIB}

\end{document}